\documentclass{mc}

\usepackage[
  paper=a4paper,         
  left=25mm,
  right=25mm,
  top=25mm,
  bottom=25mm,
  includeheadfoot,       
  heightrounded
]{geometry}

\usepackage[utf8]{inputenc}
\usepackage{textgreek}
\usepackage{tikz}
\renewcommand\datereceived{}
\renewcommand\dateaccepted{}
\usepackage{mathrsfs}
\usepackage{caption}
\usepackage{subcaption}
\usepackage{multirow}
\usepackage{xcolor,sectsty}
\usepackage{booktabs}

\usepackage{algorithm}
\usepackage{algpseudocode}

\usepackage{titlesec}
\titleformat{\subsection}
  {\normalsize} 
  {\thesection.\arabic{subsection}} 
  {1em}
  {\textit}

\definecolor{astral}{RGB}{46,116,181}
\definecolor{darkslategray}{rgb}{0.18, 0.31, 0.31}
\definecolor{warmblack}{rgb}{0.0, 0.46, 0.36}
\sectionfont{\color{black}\normalsize}

\renewcommand{\thetable}{\thesection.\arabic{table}} 

\usepackage{hyperref}

\begin{document}
	
	
	\title{{An adaptive wavelet-based PINN for problems with localized high-magnitude source}}
	
	



    

    \author[Pandey {\em et al.}]{ {\bf Himanshu Pandey}\affil{1} and {\bf Ratikanta Behera}\affil{1}}
    
    \address{\affilnum{1}{\scriptsize Department of Computational and Data Sciences, Indian Institute of Science, Bangalore, 560012, India}} 
	
	\emails{{\tt phimanshu@iisc.ac.in} (H. Pandey), {\tt ratikanta@iisc.ac.in} (R. Behera)}
	%
	

\begin{abstract}
In recent years, physics-informed neural networks (PINNs) have gained significant attention for solving differential equations, although they suffer from two fundamental limitations, namely, spectral bias inherent in neural networks and loss imbalance arising from multiscale phenomena. This paper proposes an adaptive wavelet-based PINN (AW-PINN) to address the extreme loss imbalance characteristic of problems with localized high-magnitude source terms. Such problems frequently arise in various physical applications, such as thermal processing, electro-magnetics, impact mechanics, and fluid dynamics involving localized forcing. The proposed framework dynamically adjusts the wavelet basis function based on residual and supervised loss. This adaptive nature makes AW-PINN handle problems with high-scale features effectively without being memory-intensive. Additionally, AW-PINN does not rely on automatic differentiation to obtain derivatives involved in the loss function, which accelerates the training process. The method operates in two stages, an initial short pre-training phase with fixed bases to select physically relevant wavelet families, followed by an adaptive refinement that adapts scales and translations without populating high-resolution bases across entire domains. Theoretically, we show that under certain assumptions, AW-PINN admits a Gaussian process limit and derive its associated NTK structure. We evaluate AW-PINN on several challenging PDEs featuring localized high-magnitude source terms with extreme loss imbalances having ratios up to $10^{10}:1$. Across these PDEs, including transient heat conduction, highly localized Poisson problems, oscillatory flow equations, and Maxwell’s equations with a point charge source, AW-PINN consistently outperforms existing methods in its class.\\
\end{abstract}
	
	
\maketitle	

\section{Introduction}\label{sec: Intro}

Recent years have witnessed significant progress in learning-based methods, mainly due to enhanced computational capabilities, advanced optimization algorithms, and the development of automatic differentiation techniques \cite{JMLR:autograd, NEURIPS2019_ADpytorch}. These advances have facilitated the widespread adoption of deep learning approaches within the scientific community. Among these, physics-informed neural networks (PINNs) \cite{RAISSI2019686} have emerged as a promising tool for solving partial differential equations (PDEs) in both forward and inverse settings. PINNs provide a meshless framework that optimizes the loss function by incorporating physical information through the residuals of differential equations. This methodology offers a robust alternative to conventional numerical methods, especially for tackling high-dimensional PDEs \cite{NNHighDim} and problems defined on complex geometries \cite{SAHLICOSTABAL2024_complexgeometry}. In particular, the capability of PINNs stands out when addressing inverse problems and generating highly resolved solutions by generalizing solutions across a complete computational domain rather than discrete points. Additionally, the PINN framework has been successfully extended to address integro-differential equations \cite{YUAN2022111260_integroDiff}, stochastic differential equations \cite{SIAM_stochastic, TUSHAR2023112004_stchastic}, and fractional differential equations \cite{SIAM_fractional, SM2024_fractional_optimization}. For a comprehensive overview of the theory and applications of PINNs, refer to \cite{Cuomo2022_review, zhang2024physics_review, physicsoffluids_review} and the references therein.

Although PINNs offer notable advantages, several limitations have been extensively discussed and addressed in the literature. One key limitation arises from the spectral bias inherent in neural networks \cite{pmlr-v97-rahaman19a_spectralbias, Xu2025_spectral_bias}, which causes PINNs to preferentially learn low-frequency components of the solution. Jacot {\em et al.} \cite{NEURIPS2018_5a4be1fa_jacot} formulated the neural tangent kernel (NTK) theory, demonstrating that the training process of an infinite-width neural network is equivalent to kernel regression with the NTK. Their analysis shows that loss function components associated with larger eigenvalues converge more rapidly, and the eigenvalues decrease sharply as the frequency of the target function increases. Building on NTK theory, Wang {\em et al.} \cite{WANG2022110768_PINNNTK} investigated the training dynamics of PINNs and confirmed the presence of spectral bias in PINNs. Additionally, the work by Wang {\em et al.} \cite{wang2021understanding_gradientflow} revealed that the gradient flow in PINN models becomes stiff for PDEs with high-frequency components, resulting in imbalanced gradients during backpropagation. To overcome the challenge of learning high-frequency components, Tancik {\em et al.} \cite{NEURIPS2020_55053683_FF} proposed Fourier feature mapping, which enhances the network's ability to represent high-frequency functions. Building on this approach, Wang {\em et al.} \cite{WANG2021113938_eigen_FF} demonstrated that Fourier feature mapping modulates the frequency of NTK eigenvectors and developed an improved PINN architecture by integrating Fourier features before the fully connected layers. This architecture enhances the capacity of PINNs to solve PDEs with pronounced high-frequency characteristics. Further advancements include the work of Liu {\em et al.} \cite{LIU2025106886_adaptive_FF}, which used multiple Fourier feature encoding and adaptively adjusted encoding frequencies based on the loss function.

In addition to spectral bias, PINNs also encounter poor training dynamics when there is a significant magnitude disparity between different loss terms. This loss imbalance is prominent in multiscale PDEs, where the optimizer tends to minimize the most dominant loss term at the expense of others. Such an imbalance hinders the ability of the model to accurately learn and represent all relevant features of the solution. Several studies have been conducted to address the loss balancing issue in PINNs. McClenny {\em et al.} proposed Self-adaptive PINNs (SA-PINNs) \cite{MCCLENNY2023111722_SA} that use trainable self-adaptive weights in loss terms at each training point. During optimization, the loss is minimized with respect to the network weights and maximized with respect to the self-adaptive weights. A non-negative, strictly increasing, and differentiable mask function, which takes self-adaptive weights as an argument, is used to make the network pay more attention to training points that are hard to fit. To further address the issue of loss term imbalance, Wang {\em et al.} \cite{WANG2024113112_MM} developed a modified loss function by imposing different power operations on each loss term to make them of the same order of magnitude. Further, Bischof {\em et al.} \cite{BISCHOF2025117914_ReLoBR} developed relative loss balancing with random look-back, in which loss weights are updated based on loss statistics from previous training steps using exponential decay. Additionally, a Bernoulli random variable is used to decide whether relative improvements since the beginning of the training should be carried forward. This random look-back helps the model escape local minima by changing the loss space. These loss balancing techniques require careful selection of hyperparameters or prior knowledge of the solution to appropriately scale the loss terms for optimal performance. Alternatively, rather than explicitly balancing the loss function, Pandey {\em et al.} introduced wavelet-based PINNs (W-PINNs) \cite{pandey2025efficientwaveletbasedphysicsinformedneural}, which train the model in wavelet space and bring it back to physical space using precomputed wavelet matrices. In wavelet space, the multiscale nature of the problem is smoothed out, making the optimization effective. Eventually, it transforms the optimization process to learn weights corresponding to wavelet bases composed of various scales and translates, inherently addressing both loss balancing and the spectral bias nature of the PINN. Furthermore, W-PINNs do not rely on automatic differentiation for derivatives involved in the loss function, which significantly accelerates the training process. However, this spectral approach is promising, but managing wavelet matrices, particularly for problems involving extremely high-scale features, becomes memory-intensive. Moreover, the exponential growth in the number of wavelet basis functions with increasing scales further increases optimization challenges.  

In the present study, we propose an adaptive wavelet-based PINN (AW-PINN), which extends the W-PINN framework by enabling the network to dynamically adapt the wavelet family, rather than relying on a fixed basis. This adaptive nature enhances the effectiveness of the method in resolving problems with high-scale features. In particular, we focus on problems with localized high-magnitude source terms. Such problems frequently arise in various physical contexts, including heat conduction in materials with spatially localized sources, like in welding, electromagnetic wave propagation in inhomogeneous media, and fluid dynamics involving localized forcing. The rest of the manuscript is organized into four sections. Starting with Section \ref{sec: Related}, which reviews several related methodologies, we later compare the performance of our method against these. Followed by Section \ref{sec: AW-PINN} describes the proposed AW-PINN framework in detail, and subsequently, the testing over a couple of challenging problems in Section \ref{sec: Results}. Finally, Section \ref{sec: conclude} summarizes the main findings along with the method’s limitations, and potential directions for future research.


\section{Related Works} \label{sec: Related}

\subsection{PINN framework for problems with multi-magnitude loss terms}\label{MMPINN}

Wang {\em et al.} \cite{WANG2024113112_MM} developed a PINN framework for multiscale problems with multi-magnitude loss terms (MMPINN). To balance the magnitude of different loss terms, they proposed a regularization strategy that applies power operations on each loss term. To further refine their strategy, they incorporate multi-level training by adjusting the power exponent at each level and using a grouping regularization for different subdomains. The effectiveness of this regularization scheme is demonstrated through several PDE examples exhibiting loss imbalance, including a few with strong localized source terms. This makes it a fair comparison benchmark for our proposed method.

Consider the following general form of PDE
\begin{equation}
    \begin{cases}
        \mathscr{P}[u(\boldsymbol{x})] = f(\boldsymbol{x}), \quad \boldsymbol{x}\in \Omega,\\[4pt]
        \mathscr{B}[u(\boldsymbol{x})] = g(\boldsymbol{x}), \quad \boldsymbol{x} \in \partial \Omega,
    \end{cases}
    \label{eq:1-PDE}
\end{equation}
where $\mathscr{P}[\cdot]$ represents a differential operator on the $d$-dimensional domain $\Omega \subset \mathbb{R}^d$, which also includes the time variable as a component, $f$ is the source term, and $\mathscr{B}$ corresponds to the given initial and boundary conditions. Let $\hat u(\boldsymbol{x};\boldsymbol{\theta})$ represent the PINN-based approximation to the PDE \ref{eq:1-PDE}, with $\boldsymbol{\theta}$ denoting the trainable parameters of the model. These parameters, which include weights and biases of the network, are obtained by optimizing the loss function. For the traditional PINN, the loss function can be expressed as follows 

\begin{equation}
    \begin{cases}
        \mathcal{L}(\boldsymbol{\theta})
        = \omega_r\,\mathcal{L}_{res}(\boldsymbol{\theta}) + \omega_s\,\mathcal{L}_{sup}(\boldsymbol{\theta}),\\[6pt]
        
        \mathcal{L}_{res}(\boldsymbol{\theta})
        = \dfrac{1}{N_{res}}\displaystyle\sum_{i=1}^{N_{res}}\bigl\lvert \mathscr{P}\!\left[\hat u(\boldsymbol{x}_i^r;\boldsymbol{\theta})\right] - f(\boldsymbol{x}_i^r) \bigr\rvert^2,\quad x_i^r\in\Omega,\\[6pt]
        
        \mathcal{L}_{sup}(\boldsymbol{\theta})
        = \dfrac{1}{N_{sup}}\displaystyle\sum_{i=1}^{N_{sup}}\bigl\lvert \mathscr{B}\!\left[\hat u(\boldsymbol{x}_i^s;\boldsymbol{\theta})\right] - g(\boldsymbol{x}_i^s) \bigr\rvert^2, \quad x_i^s \in \partial \Omega,\\
    \end{cases}
    \label{eq:2-PINN-Loss}
\end{equation}
where the total loss $\mathcal{L}$ is composed of the residual loss, $\mathcal{L}_{res}$, and the supervised loss $\mathcal{L}_{sup}$. $N_{res}$ and $N_{sup}$ are the numbers of residual and supervised training points in $\Omega$ and $\partial \Omega$, respectively. This loss function is generally minimized using the most popular optimization algorithms, namely, Adam \cite{kingma2017adammethodstochasticoptimization} and L-BFGS \cite{LBFGS}. 

For highly imbalanced loss terms, MMPINN proposes an exponent regularized loss function, for which the loss function takes the following form
\begin{equation}
    \mathcal{\bar L}(\boldsymbol{\theta})
        = \mathcal{L}_{res}^{\frac{1}{p}}(\boldsymbol{\theta}) + \mathcal{L}_{sup}^{\frac{1}{q}}(\boldsymbol{\theta}), \quad p,q\in \mathbb{Z}^+, 
    \label{eq:loss}
\end{equation}
where $p$ and $q$ are positive integer regularization parameters, which are chosen in a way that balances the order of magnitude of different loss terms. However, the effectiveness of this regularization depends on choosing $p$ and $q$ within an optimal range that is problem-specific.

\subsection{Wavelet-based PINN}\label{W-PINN}

The Wavelet-based PINN (W-PINN) \cite{pandey2025efficientwaveletbasedphysicsinformedneural} learns the solution in wavelet space spanned by wavelet basis functions at multiple scales and translations that cover the entire domain. The learned solution is then mapped back to physical space using predefined wavelet basis transformation matrices. Let
\(
\Omega=\prod_{n=1}^{d}[a_n,b_n]\subset\mathbb{R}^d
\)
be a compact domain. For a mother wavelet \(\psi\) on \(\mathbb{R}\), the scaled and translated wavelets can be defined as
\begin{equation}
    \psi_{j,k}(x)\;=\;2^{j/2}\,\psi\!\left(2^{j}x-k\right),
    \qquad j,k\in\mathbb{Z},
    \label{eq:wav-1d}
\end{equation}
where $j$ and $k$ are scale and translation parameters. This can be extended in $d$-dimensions by using a separable tensor-product basis in the following manner
\begin{equation}
    \Psi_{\boldsymbol{j},\boldsymbol{k}}(\boldsymbol{x})
    =\prod_{n=1}^{d} 2^{j_n/2}\,\psi\left(2^{j_n}x_n-k_n\right),
    \qquad
    \boldsymbol{x} \in \Omega,~~ \boldsymbol{j}=(j_1,\dots,j_d),\ \boldsymbol{k}=(k_1,\dots,k_d)\in\mathbb{Z}^d. 
    \label{eq:wav-dd}
\end{equation}

The scale parameters are chosen from a finite resolution-level sets \(J_n=\{J_{n,1},\dots,J_{n,N_n}\}\subset\mathbb{Z}\) for each coordinate \(n=1,\dots,d\). For each \(j_n\in J_n\), define the integer translation range 

\begin{equation}
    K_n(j_n)\;:=\;\Big\{\,k\in\mathbb{Z}\ :\ 
    \big\lfloor (a_n-\gamma)\,2^{j_n}\big\rfloor\ \le\ k\ \le\ \big\lceil (b_n+\gamma)\,2^{j_n}\big\rceil
    \,\Big\},
    \label{eq:trans-range}
\end{equation}
which ensures coverage of \(\Omega\) at each resolution. Here \(\gamma>0\) is a fixed translation hyperparameter. For this hierarchical multi-resolution basis, we can define a multi-index set as
\begin{equation}
    \mathcal{I}\;=\;\Big\{(\boldsymbol{j},\boldsymbol{k})\in\mathbb{Z}^{d}\times\mathbb{Z}^{d}\ :\ j_n\in J_n,\ k_n\in K_n(j_n)\ \text{for } n=1,\dots,d\Big\}.
\end{equation}

Consider a bijection mapping \(\mathcal{M}:\mathcal{I}\to\{1,\dots,N_{\mathrm{fam}}\}\), and define
\(\Psi_{i}(\boldsymbol{x}):=\Psi_{\boldsymbol{j},\boldsymbol{k}}(\boldsymbol{x})\) for \(i=\mathcal{M}(\boldsymbol{j},\boldsymbol{k})\). Let $\hat u^\psi(\boldsymbol{x};\boldsymbol{\theta})$ is a W-PINN approximation to the PDE \ref{eq:1-PDE} such that
\begin{equation}
        \hat u^\psi(\boldsymbol{x};\boldsymbol{\theta}) = \displaystyle\sum_{i=1}^{N_{fam}} c_i(\boldsymbol{\theta})\Psi_i(\boldsymbol{x}) + \mathcal{B}, \quad \boldsymbol{x}\in\Omega\cup\partial \Omega,
        \label{eq:WPINN-app}
\end{equation}
where \(\boldsymbol{\theta}\cup\{\mathcal{B}\}\) are the trainable parameters of W-PINN network for learnable coefficients \(\{c_i\}_{i=1}^{N_{\mathrm{fam}}}\), and the fixed basis \(\{\Psi_i\}_{i=1}^{N_{\mathrm{fam}}}\). Training minimizes the loss in Equation \eqref{eq:2-PINN-Loss}, with \(\hat u\) replaced by \(\hat u^{\psi}\) from \eqref{eq:WPINN-app}. Instead of using autograd, derivatives required in the loss function are obtained by analytically differentiating the basis functions \(\Psi_i\), i.e,
\begin{equation}
    \partial_{x_m}\Psi_{\boldsymbol{j},\boldsymbol{k}}(\boldsymbol{x})
    = 2^{j_m}\,2^{j_m/2}\,\psi'\left(2^{j_m}x_m-k_m\right)
    \times \prod_{n\neq m} 2^{j_n/2}\,\psi\left(2^{j_n}x_n-k_n\right),
    \label{eq:derv}
\end{equation}
and higher-order derivatives follow similarly. The W-PINN effectively addresses multiscale PDEs because the localized wavelet basis transforms the problem into wavelet space, where multiscale features become smoother, eventually facilitating the training process.

\subsection{Neural Tangent Kernel theory for PINN}\label{NTK}

The neural tangent kernel (NTK) \cite{NEURIPS2018_5a4be1fa_jacot, WANG2022110768_PINNNTK} provides a theoretical perspective for training dynamics of neural networks. Consider an infinite width PINN approximating the PDE \eqref{eq:1-PDE} with $\hat u(\boldsymbol{x};\boldsymbol{\theta})$.  For gradient descent algorithm with a sufficiently small learning rate, Wang {\em et al.} \cite{WANG2022110768_PINNNTK} showed that, in the infinite-width limit, the training dynamics converge to independent, identically distributed Gaussian processes with zero mean. Based on this, NTK for PINN is defined as
\begin{equation}
    \boldsymbol{\mathcal K}(t) = 
    \begin{bmatrix}
    \boldsymbol{\mathcal K}_{pp}(t) & \boldsymbol{\mathcal K}_{pb}(t) \\
    \boldsymbol{\mathcal K}_{bp}(t) & \boldsymbol{\mathcal K}_{bb}(t)
    \end{bmatrix},
\end{equation}
where $t$ is the training step. For given training points $\{ \boldsymbol{x}_i^{r}, f (\boldsymbol{x}_i^{r})\}$ and $\{ \boldsymbol{x}_i^{s}, g (\boldsymbol{x}_i^{s})\}$, $\{i,j\}$-th element of submatrices of the NTK matrix $ \boldsymbol{\mathcal K}$ is obtained by

\begin{align*}
\left(\boldsymbol{\mathcal{K}}_{pp}\right)_{ij}(t) &= 
\left\langle 
\frac{ \partial \mathscr{P} \left[ \hat u\left( \boldsymbol{x}_i^{r}; \boldsymbol{\theta}(t) \right) \right] }{ \partial\boldsymbol{\theta} },\ 
\frac{ \partial \mathscr{P} \left[ \hat u\left( \boldsymbol{x}_j^{r}; \boldsymbol{\theta}(t) \right) \right] }{ \partial\boldsymbol{\theta} }
\right\rangle, \\
\left(\boldsymbol{\mathcal{K}}_{pb}\right)_{ij}(t) &= 
\left\langle
\frac{ \partial \mathscr{P} \left[ \hat u\left( \boldsymbol{x}_i^{r}; \boldsymbol{\theta}(t) \right) \right] }{ \partial\boldsymbol{\theta} },\ 
\frac{ \partial \mathscr{B} \left[ \hat u\left( \boldsymbol{x}_j^{s}; \boldsymbol{\theta}(t) \right) \right] }{ \partial\boldsymbol{\theta} }
\right\rangle, \\
\left(\boldsymbol{\mathcal{K}}_{bb}\right)_{ij}(t) &= 
\left\langle
\frac{ \partial \mathscr{B} \left[ \hat u\left( \boldsymbol{x}_i^{s}; \boldsymbol{\theta}(t) \right) \right] }{ \partial\boldsymbol{\theta} },\ 
\frac{ \partial \mathscr{B} \left[ \hat u\left( \boldsymbol{x}_j^{s}; \boldsymbol{\theta}(t) \right) \right] }{ \partial\boldsymbol{\theta} }
\right\rangle,
\end{align*}
where $\langle\cdot,\cdot\rangle$ denotes the inner product over all trainable parameters. Further, using the NTK matrix, training a PINN through gradient descent can be expressed as
\begin{equation}
\left[
\begin{array}{c}
\dfrac{\partial \mathscr{P}\left[\hat u\left(\boldsymbol{x}^{r};\,\boldsymbol{\theta}(t)\right)\right]}{\partial t}
\\[1em]
\dfrac{\partial \mathscr{B}\left[\hat u\left(\boldsymbol{x}^{s};\,\boldsymbol{\theta}(t)\right)\right]}{\partial t}
\end{array}
\right]
=
-\boldsymbol{\mathcal{K}}(t)
\left[
\begin{array}{c}
\mathscr{P}\left[\hat u\left(\boldsymbol{x}^{r};\,\boldsymbol{\theta}(t)\right)\right]-f\left(\boldsymbol{x}^{r}\right)
\\
\mathscr{B}\left[\hat u\left(\boldsymbol{x}^{s};\,\boldsymbol{\theta}(t)\right)\right]-g\left(\boldsymbol{x}^{s}\right)
\end{array}
\right].
\label{eq:NTK_eq}
\end{equation}

Under the assumption of training a sufficiently large-width PINN with a small enough learning rate, Wang {\em et al.} \cite{WANG2022110768_PINNNTK} showed the NTK remains constant, ie, $\boldsymbol{\mathcal{K}}(t)\approx\boldsymbol{\mathcal{K}}(0)$. Using this solution of the Equation \eqref{eq:NTK_eq} becomes
\begin{equation}
\left[
\begin{array}{c}
\mathscr{P}\left[\hat u\left(\boldsymbol{x}^{r};\,\boldsymbol{\theta}(t)\right)\right] - f\left(\boldsymbol{x}^{r}\right)
\\
\mathscr{B}\left[\hat u\left(\boldsymbol{x}^{s};\,\boldsymbol{\theta}(t)\right)\right] - g\left(\boldsymbol{x}^{s}\right)
\end{array}
\right]
\approx
- e^{-\boldsymbol{\mathcal{K}}(0) t}
\left[
\begin{array}{c}
f\left(\boldsymbol{x}^{r}\right) 
\\
g\left(\boldsymbol{x}^{s}\right)
\end{array}
\right],
\label{eq:NTK_sol}
\end{equation}
which provides error evolution during training, which can be further simplified by spectral decomposition of the NTK: $\boldsymbol{\mathcal{K}}(0) = \boldsymbol{\mathcal{Q}}^T \boldsymbol{\Lambda} \boldsymbol{\mathcal{Q}}$, here $\boldsymbol{\mathcal{Q}}$ and $\boldsymbol{\Lambda}$ are orthogonal matrix and corresponding diagonal eigenvalue matrix. Then the Equation \eqref{eq:NTK_sol} take the form

\begin{equation}
\left[
\begin{array}{c}
\mathscr{P}\left[\hat u\left(\boldsymbol{x}^{r};\,\boldsymbol{\theta}(t)\right)\right] - f\left(\boldsymbol{x}^{r}\right)
\\
\mathscr{B}\left[\hat u\left(\boldsymbol{x}^{s};\,\boldsymbol{\theta}(t)\right)\right] - g\left(\boldsymbol{x}^{s}\right)
\end{array}
\right]
\approx
- \boldsymbol{\mathcal{Q}}^T e^{-\boldsymbol{\Lambda} t} \boldsymbol{\mathcal{Q}}
\left[
\begin{array}{c}
f\left(\boldsymbol{x}^{r}\right) 
\\
g\left(\boldsymbol{x}^{s}\right)
\end{array}
\right].
\label{eq:NTK_sol}
\end{equation}

This formulation suggests that the $i^{th}$ error component exhibits exponential decay at a rate determined by the corresponding NTK eigenvalue; hence, larger eigenvalues yield faster decay of the error and, consequently, faster convergence of the PINN.


\section{Adaptive Wavelet-based PINN}\label{sec: AW-PINN}

The Adaptive Wavelet-based PINN (AW-PINN) extends the W-PINN framework by making the wavelet family adaptive, allowing scales and translates to be learned as parameters, unlike the fixed values in W-PINN. The method consists of two stages: a short pre-training stage using the W-PINN formulation \ref{W-PINN}, followed by an adaptive stage that refines the selected wavelet families obtained from the pre-training stage. We discuss both stages in detail below.

\begin{itemize}
    \item \emph{Pre-training and family selection}

Consider W-PINN pre-training for the PDE \eqref{eq:1-PDE} with residual points $\{x^r_i\}_{i=1}^{N_{res}}$ and supervised points $\{x^s_i\}_{i=1}^{N_{sup}}$. Evaluate the right side of PDE \eqref{eq:1-PDE} at the residual and supervised training points in vector form

\begin{equation}
\begin{aligned}
    \mathbf{f} &= (f(\boldsymbol{x}_1^r),\ldots,f(\boldsymbol{x}_{N_{res}}^r))^\top \in \mathbb{R}^{N_{res}},\\  
    \mathbf{g} &= (g(\boldsymbol{x}_1^s),\ldots,g(\boldsymbol{x}_{N_{res}}^s))^\top \in \mathbb{R}^{N_{sup}},
\label{eq:rhs}
\end{aligned}
\end{equation}
where $f$ and $g$ denote the PDE residual source term and supervised condition values, respectively. For each wavelet family index $i$ with basis $\Psi_i$ and pre-trained coefficient $c_i$, define the family-induced residual and boundary response vectors
\begin{equation}
\begin{aligned}
    \mathbf{R}_i &= \big(\mathscr{P}[c_i\Psi_i(\boldsymbol{x}_1^r),\ldots,\mathscr{P}[c_i\Psi_i(\boldsymbol{x}_{N_{res}}^r) \big)^\top \in\mathbb{R}^{N_{res}}, \\ 
    \mathbf{B}_i &= \big(\mathscr{B}[c_i\Psi_i(\boldsymbol{x}_1^s),\ldots,\mathscr{B}[c_i\Psi_i(\boldsymbol{x}_{N_{sup}}^s) \big)^\top \in\mathbb{R}^{N_{sup}},
    \label{eq:RB}
\end{aligned}
\end{equation}
then we get the alignment measure between the family $i$ and the PDE right side via the following similarity scores

\begin{equation}
    \text{Score}_i^{R} =
    \frac{\langle \mathbf{R}_i,\ \mathbf{f}\rangle}
         {|\langle \mathbf{R}_i,\ \mathbf{f}\rangle|_{max}}, \qquad
    \text{Score}_i^{B} =
    \frac{\langle \mathbf{B}_i,\ \mathbf{g}\rangle}
         {|\langle \mathbf{B}_i,\ \mathbf{g}\rangle|_{max}},
    \label{eq:similarity}
\end{equation}
When the right-hand side of the PDE is identically zero, the above similarity becomes ill-defined. In such cases, mean of the normalized response vector, $\mathbf{\bar R}_i / \|\mathbf{R}_i\|_{max}$ and $\mathbf{\bar B}_i / \|\mathbf{B}_i\|_{max}$, are used as respective scores. Wavelet families exhibiting a higher similarity score or lower normalized response magnitudes are deemed more physically relevant. Additionally, wavelet families corresponding to the top $\kappa$ largest coefficient magnitudes $|c_i|$ are also included. This similarity-based selection works because the wavelet bases exhibit minimal overlap. The final set of active families is denoted as
\begin{equation}
    \mathcal{I}_{\mathrm{A}} = 
    \Big\{\,(\boldsymbol{j},\boldsymbol{k}) \in \mathbb{Z}^{d}\times\mathbb{Z}^{d}\,\Big\} 
    \subset \mathcal{I},
\end{equation}
along with their corresponding coefficients
\(\boldsymbol{c}_{\mathrm{A}} = \{ c_i \}_{i=1}^{N_{\mathrm{A}}}\), $N_{\mathrm{A}}$ denotes number of selected wavelet basis function.

\begin{figure}[t!] 
    \centering
    \includegraphics[width=1.0\textwidth]{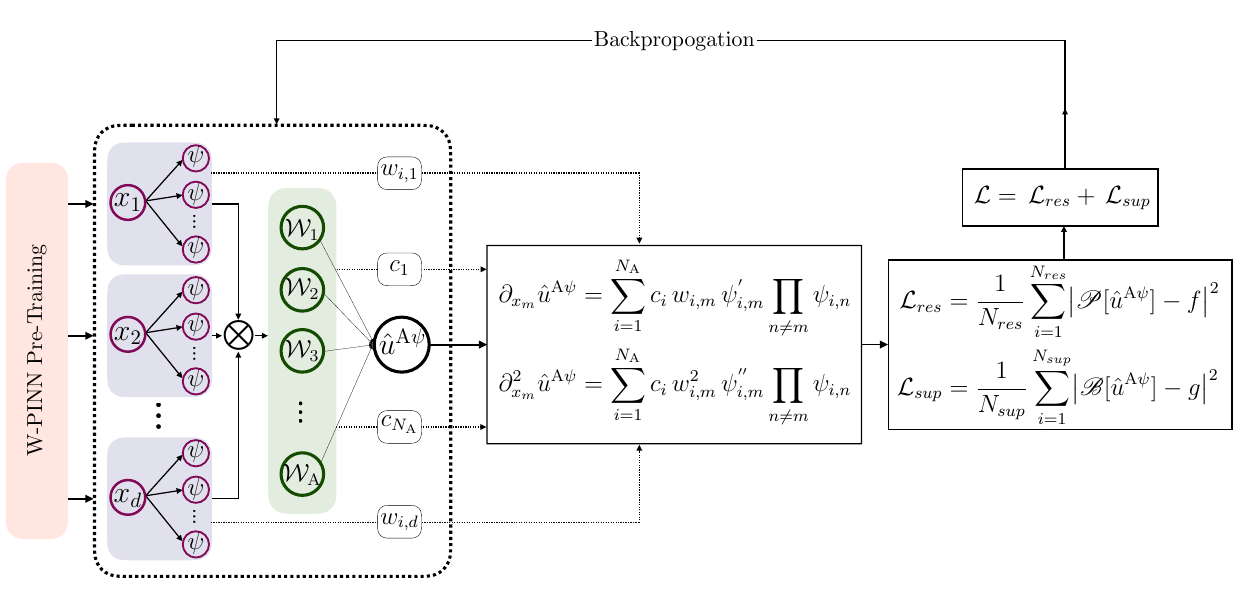} 
    \caption{A schematic architecture of AW-PINN. Parameters for purple wavelet units and final linear layer are initialized using $\mathcal{I_{\text{adapt}}}$ and $\boldsymbol{c}_{\text{adapt}}$, respectively. The symbol $\psi$ denotes the wavelet activation function and $\psi_{i,m} = \psi(w_{i,m}x_m+b_{i,m})$, while $\bigotimes$ represents element-wise multiplication.}   
    \label{Fig:Network}
\end{figure}

\begin{algorithm}[b!]
\caption{Adaptive Wavelet-based PINN (AW-PINN)}
\label{alg:awpinn}
\begin{algorithmic}[1]

\Statex \textbf{Input:} $\mathscr{P},\mathscr{B}$, domain $\Omega$, points $\{x_i^r\},\{x_i^s\}$, wavelet resolutions $\{J_n\}$, translation hyperparameter $\gamma$, threshold $\kappa$.
\vspace{0.2em}
\Statex \textbf{Stage 1: Pre-training and family selection}
\State Build fixed family $\mathcal{I}=\{(\mathbf{j},\mathbf{k}): j_n\in J_n,\; k_n\in K_n(j_n)\}$, as described in \cite{pandey2025efficientwaveletbasedphysicsinformedneural}.
\State Pre-train W-PINN $\hat u^\psi(\boldsymbol{x};\boldsymbol{\theta}) = \sum c_i(\boldsymbol{\theta})\Psi_i(\boldsymbol{x}) + \mathcal{B}$ via loss \eqref{eq:2-PINN-Loss}.
\State For each $i\in\mathcal{I}$ compute responses $\mathbf{R}_i=\mathscr{P}[c_i\Psi_i]$, $\mathbf{B}_i=\mathscr{B}[c_i\Psi_i]$.
\State Compute similarity scores
$\text{Score}^R_i=\langle\mathbf{R}_i,\mathbf{f}\rangle / |\langle\mathbf{R}_i,\mathbf{f}\rangle|_{\max}$,
$\text{Score}^B_i=\langle\mathbf{B}_i,\mathbf{g}\rangle / |\langle\mathbf{B}_i,\mathbf{g}\rangle|_{\max}$.
\State Select active families 
$\mathcal{I}_A:=\text{selection based on score threshold and top }\kappa\text{ coefficients }|c_i|$.
\vspace{0.2em}
\Statex \textbf{Stage 2: Adaptive stage}
\State Initialize adaptive parameters for $i\in\mathcal{I}_A$: 
$w_{i,n}=2^{j_{i,n}},\; b_{i,n}=-k_{i,n}$.
\State Define adaptive wavelet unit $\mathcal{W}_i(\boldsymbol{x};\boldsymbol{\theta}_i)=\prod_{n=1}^d \psi(w_{i,n}x_n+b_{i,n})$.
\State Initialize $c_i$ from pre-training and construct
$\hat u^{\mathrm{A}\psi}(\boldsymbol{x};\boldsymbol{\theta})=\sum_{i\in\mathcal{I}_A} c_i(\boldsymbol{\theta})\,\mathcal{W}_i(\boldsymbol{x};\boldsymbol{\theta}_i) + \mathcal{B}$.
\State Compute analytic derivatives of $\mathcal{W}_i$ for $\mathscr{P}$ and $\mathscr{B}$.
\State Train model parameters $\boldsymbol{\theta}=\{c_i,\boldsymbol{\theta}_i,\mathcal{B}\}$ by minimizing loss in Eq.~\eqref{eq:2-PINN-Loss} using L-BFGS until convergence.

\Statex \textbf{Output:} Trained solution \(\hat u^{A\psi}(\boldsymbol{x}; \boldsymbol{\theta^*})\)

\end{algorithmic}
\end{algorithm}

\item \emph{Adaptive stage and network formulation}

The selected coefficients and scale-translate parameters are utilized to initialize the AW-PINN network. For each selected family $i\in\mathcal{I}_{\mathrm{A}}$ associated with scale–translate pair $(\boldsymbol{j}_i,\boldsymbol{k}_i)$, an adaptive submodule $\mathcal{W}_i(\boldsymbol{x};\boldsymbol{\theta}_i)$ is constructed as a parametric wavelet. For $d$-dimensional input $\boldsymbol{x}=(x_1,\dots,x_d)$, we define
\begin{equation}
    \mathcal{W}_i(\boldsymbol{x};\boldsymbol{\theta}_i)
    = \prod_{n=1}^{d}
    \psi\!\left(w_{i,n}\,x_n + b_{i,n}\right),
    \qquad
    \boldsymbol{\theta}_i = \{(w_{i,n},b_{i,n})\}_{n=1}^{d},
    \label{eq:wavelet-unit}
\end{equation}
where $\psi(\cdot)$ is the same mother wavelet used in the Equation \eqref{eq:wav-1d}. The initial parameters are transferred from \(\mathcal{I_{\mathrm{A}}}\) as, \(w_{i,n}^{(0)} = 2^{j_{i,n}}, b_{i,n}^{(0)} = -k_{i,n},\) enabling continuous adaptation of scale and translation during optimization. This adaptive wavelet layer is achieved by using \(\psi(\cdot)\) as activation function. A final linear layer, initialized with \(\boldsymbol{c}_{\mathrm{A}}\), is imposed to get AW-PINN approximation

\begin{equation}
    \hat u^{\mathrm{A}\psi}(\boldsymbol{x};\boldsymbol{\theta})
    = \sum_{i=1}^{N_{\mathrm{A}}}
      c_i(\boldsymbol{\theta})\,\mathcal{W}_i(\boldsymbol{x};\boldsymbol{\theta}_i)
      + \mathcal{B},
    \label{eq:AWPINN-app}
\end{equation}
where \(\mathcal{B}\) denotes a trainable bias term. The optimization process minimizes the same physics-informed loss as in the Equation \eqref{eq:2-PINN-Loss}, and the autograd is avoided by taking the analytical derivative of the Equation \eqref{eq:wavelet-unit}, similar to the Equation \eqref{eq:derv}. A schematic diagram of the AW-PINN architecture is shown in Fig. \ref{Fig:Network}, and the AW-PINN algorithm is summarized in Algorithm \ref{alg:awpinn}.

\end{itemize}

We now demonstrate that, as the adaptive family grows to an infinitely large size, under certain assumptions, the AW-PINN training converges to a zero-mean Gaussian process \(\mathcal{G}_p\).

\begin{theorem}
    For bias $\mathcal{B}=0$, as $N_{\mathrm{A}}\to\infty$, the random functions $\hat u_N^{\mathrm{A}\psi}(\cdot)$ converge in finite-dimensional distributions to a centered Gaussian process
\[
\hat u_\infty^{\mathrm{A}\psi}(\cdot)\ = \frac{1}{\sqrt {N_\mathrm{A}}} \sum_{i=1}^{N_{\mathrm{A}}}
      c_i\mathcal{W}_i\sim\ \mathcal{G}_p\!\big(0,k(\cdot,\cdot)\,\big)
\]
\end{theorem}

\begin{proof} Fix $m\in\mathbb{N}$ and inputs $X=\{\boldsymbol{x}^{(1)},\ldots,\boldsymbol{x}^{(m)}\}\subset\mathbb{R}^d$. Define
\[
\mathbf{f}_{N_{\mathrm{A}}} \;=\; \big(\hat u_{N_{\mathrm{A}}}^{\mathrm{A}\psi}(\boldsymbol{x}^{(1)}),\ \ldots,\ \hat u_{N_{\mathrm{A}}}^{\mathrm{A}\psi}(\boldsymbol{x}^{(m)})\big)^\top
=\frac{1}{\sqrt{N_{\mathrm{A}}}}\sum_{i=1}^{N_{\mathrm{A}}} c_i\,\boldsymbol{\Phi}_i,
\]
where the random feature vector
\[
\boldsymbol{\Phi}_i \;=\; \big(\mathcal{W}_i(\boldsymbol{x}^{(1)};\boldsymbol{\theta}_i),\ldots,
\mathcal{W}_i(\boldsymbol{x}^{(m)};\boldsymbol{\theta}_i)\big)^\top\in\mathbb{R}^m.
\]

Assume parameters and coefficients are initialized i.i.d. normal distribution, $c_i \stackrel{\text{i.i.d.}}{\sim}\mathcal{N}(0,\sigma_c^2),~\theta_{i}\stackrel{\text{i.i.d.}}{\sim}\mathcal{N}(0,\sigma_\theta^2),$ where $\sigma_c$ and $\sigma_\theta$ are fixed constants, then
\[
 \text{Cov}~[c_i\,\boldsymbol{\Phi}_i]=\sigma_c^2~\mathbb{E}_{\boldsymbol{\theta}}[\boldsymbol{\Phi}_i\,\boldsymbol{{\Phi}_i^{\top}}],\]
whose \((m,n)\)-th entry is \(\sigma_c^2\,\mathbb{E}_{\boldsymbol{\theta}}\,[\mathcal{W}(\boldsymbol{x^{(m)}};\boldsymbol{\theta})\mathcal{W}(\boldsymbol{x^{(n)}};\boldsymbol{\theta})].\)

By the multivariate central limit theorem, \(\mathbf{f}_{N_{\mathrm{A}}}\) converges in distribution to a multivariate Gaussian with mean zero and covariance matrix equal to that limit covariance. Because the finite dimensional distributions converge to Gaussian ones with the above covariance, the process converges to the centered \(\mathcal{G}_p\) with kernel \(\ \mathcal{K}(\boldsymbol{x},\boldsymbol{x}')=\sigma_c^2\ \mathbb E_{\boldsymbol{\theta}}\big[\mathcal W(\boldsymbol{x};\boldsymbol{\theta})\,
\mathcal W(\boldsymbol{x}';\boldsymbol{\theta})\big]\). 
\end{proof}

Although, for practical purposes, the AW-PINN parameters are transferred effectively from W-PINN pre-training, the above theorem applies under random initialization as stated. The above theorem suggests the learning dynamics of AW-PINN in the infinite family regime may be analyzed via kernel regression and neural tangent kernel theory as in Section \ref{NTK}. For a scalar output \(\hat u^{\mathrm{A}\psi}\), the empirical NTK between inputs \(x,x'\) is
\begin{equation}
\mathcal{K}(\boldsymbol{x},\boldsymbol{x'}) \;=\; \big\langle \nabla_\theta \hat u^{\mathrm{A}\psi}(\boldsymbol{x};\boldsymbol{\theta}),\; \nabla_\theta \hat u^{\mathrm{A}\psi}(\boldsymbol{x'};\boldsymbol{\theta})\big\rangle,
\end{equation}
where the inner product is over all trainable scalar parameters. Decompose the gradient into derivatives with respect to the linear coefficients \(c_i\) and the wavelet parameters \(\theta_i\). Using the model and the \(1/\sqrt{N_A}\) scaling,
\begin{align}
\frac{\partial u^{\mathrm{A}\psi}}{\partial c_i} &= \frac{1}{\sqrt{N_A}}\,\mathcal{W}_i(\boldsymbol{x};{\theta}_i),\\
\frac{\partial u^{\mathrm{A}\psi}}{\partial \theta_i} &= \frac{1}{\sqrt{N_A}}\,c_i\,\nabla_{\varphi}\mathcal{W}_i(\boldsymbol{x};{\theta}_i).
\end{align}

Hence, the NTK decomposes as a sum over units
\begin{equation}
\mathcal{K}(\boldsymbol{x},\boldsymbol{x'}) = \frac{1}{N_A}\sum_{i=1}^{N_A} \mathcal{W}_i(\boldsymbol{x}; \theta_i)\mathcal{W}_i(\boldsymbol{x'}; \theta_i)
+ \frac{1}{N_A}\sum_{i=1}^{N_A} c_i^2 \, \big\langle \nabla_\theta \mathcal{W}_i(\boldsymbol{x}; \theta_i), \, \nabla_\theta \mathcal{W}_i(\boldsymbol{x'}; \theta_i) \big\rangle,
\end{equation}
we denote the two contributions by \(\mathcal{K}_c\) (from linear weights \(c_i\)) and \(\mathcal{K}_\theta\) (from adaptive parameters):
\[
\mathcal{K}(\boldsymbol{x},\boldsymbol{x'})=\mathcal{K}_c(\boldsymbol{x},\boldsymbol{x'})+\mathcal{K}_\theta(\boldsymbol{x},\boldsymbol{x'}).
\]


\section{Results}\label{sec: Results}

This section evaluates the performance of AW-PINN on several challenging problems and compares with baseline PINN, W-PINN, and MMPINN, all of which are designed to address the loss imbalance issue in PINNs. For W-PINN and MMPINN, training starts with a few epochs of the ADAM optimizer, followed by the LBFGS optimizer until convergence. In the case of AW-PINN, the first phase is done via ADAM, and the subsequent adaptive phase employs the LBFGS optimizer. In this study, we use the Gaussian wavelet, which is the first derivative of the Gaussian function. The smooth regularity and localization property of the Gaussian wavelet make it one of the optimal choices. All model parameters are initialized using Xavier's technique \cite{pmlr-v9-glorot10a}, which prevents vanishing gradients and provides stable and efficient training. The hyperparameters used for each problem are tabulated in Appendix A \ref{A1}. For quantitative comparison, we use the relative $L_2$-Error metric over uniformly distributed test samples across the domain, which is defined as
\begin{equation}
\text{Relative $L_2$-Error}= \frac{||u(\boldsymbol{x}^t)-\hat u(\boldsymbol{x}^t;\boldsymbol{\theta})||_2}{||u(\boldsymbol{x}^t)||_2},
\end{equation}
where $u$ and $\hat u$ are the exact/reference solution and the model's prediction over test samples $\boldsymbol{x}^t$, respectively. Further, to average out the model's sensitivity over parameter initialization, each experiment is repeated five to ten times independently, and reported results correspond to the average outputs of these runs. Training is carried out using PyTorch 2.6.0 + CU12.4 on an NVIDIA RTX A6000 GPU. All source code to reproduce the results of this study will be made available on request.

\subsection{Heat conduction problem with extreme heat source}\label{problem:1}

Heat conduction problems with a localized, strong heat source are often encountered in various practical applications, such as material processing, geophysics, nuclear engineering, and many others. For such problems, challenges lie in the transient nature and the intense gradients generated in the process. This problem is chosen to test the effectiveness of the proposed method for such an extreme source.  

Consider the following heat conduction model 

{ \begin{equation}
        \begin{cases}
             \frac{\partial u}{\partial t}=\frac{\partial^2 u}{\partial x^2}+h(x,t), \quad (x,t)\in(-1,1)\times(0,1), \\[6pt]
        u(x,0)=(1-x^2)\exp\big(1/(1+\epsilon)\big), \quad x\in(-1,1), \\[6pt] 
        u(-1,t)=0,~u(1,t)=0, \quad t\in (0, 1],
        \end{cases}
    \end{equation}}
the exact solution is given as $u(x,t) = (1-x^2)\exp\big(1/{((2t-1)^2+\epsilon})\big)$, for which the source term becomes

\begin{flushleft}
\(
   \quad \quad h(x,t) = 2\Big[ 1 + 2\frac{(2t-1)(x^2-1)}{((2t-1)^2+\epsilon)^2}\Big]\exp\Big(\frac{1}{(2t-1)^2+\epsilon}\Big).
\)
\end{flushleft}

For the lower $\epsilon$, the source term as well as the solution exhibit transient behavior with a high gradient near $t=0.5$. Additionally, due to smooth initial and boundary conditions, it poses a highly loss imbalance. For instance, at $\epsilon = 0.1$, the initial loss ratio is $\mathcal{L}_b:\mathcal{L}_i:\mathcal{L}_r = 1:10:10^9$, which makes it challenging for any PINNs-based approaches.

\begin{table}[t!]
\centering
\caption{Comparison of various methods for solving the heat conduction problem~\ref{problem:1}.}
\label{Table:P1}
\small
{\begin{tabular}{cccccccc}
\hline
$\epsilon$ & Method & $N_I$ & $N_B$ & $N_R$ & Relative $L_2$-Error & Avg. Training Time \\
\hline \\

 & Baseline PINN & 5000 & 10000 & 50000 & $8.1 \pm 1.32 \times 10^{-1}$ & - \\
 & W-PINN & 1000 & 2000 & 20000 & $4.05 \pm 0.38 \times 10^{-5}$ & 15.35 min \\
 0.12   & MMPINN & 3000 & 6000 & 90000 & $1.71 \pm 0.53 \times 10^{-5}$ & 91.21 min \\
& \textbf{AW-PINN} & \textbf{1000} & \textbf{2000} & \textbf{20000} & $\mathbf{3.52 \pm 0.81 \times 10^{-6}}$ & \textbf{24.83 min} \\

\\ 

 & Baseline PINN & 5000 & 10000 & 50000 & $9.13 \pm 0.82 \times 10^{-1}$ & - \\
 & W-PINN & 1000 & 2000 & 20000 & $3.11 \pm 0.63 \times 10^{-5}$ & 16.79 min \\
 0.11   & MMPINN & 4000 & 8000 & 100000 & $2.44 \pm 0.61 \times 10^{-5}$ & 208.36 min \\
& \textbf{AW-PINN} & \textbf{1000} & \textbf{2000} & \textbf{20000} & $\mathbf{8.15 \pm 1.27 \times 10^{-6}}$ & \textbf{24.14 min} \\

\\

 & Baseline PINN & 5000 & 10000 & 100000 & $9.61 \pm 0.76 \times 10^{-1}$ & - \\
 & W-PINN & 1000 & 2000 & 20000 & $3.10 \pm 0.86 \times 10^{-5}$ & 16.28 min \\
 0.10   & MMPINN & 7000 & 14000 & 200000 & $8.46 \pm 0.68 \times 10^{-1}$ & - \\
& \textbf{AW-PINN} & \textbf{1000} & \textbf{2000} & \textbf{20000} & $\mathbf{8.86 \pm 0.87 \times 10^{-6}}$ & \textbf{25.52 min} \\

\hline
\end{tabular}
}
\end{table}

\begin{figure}[b!] 
    \centering
    \includegraphics[width=1.0\textwidth]{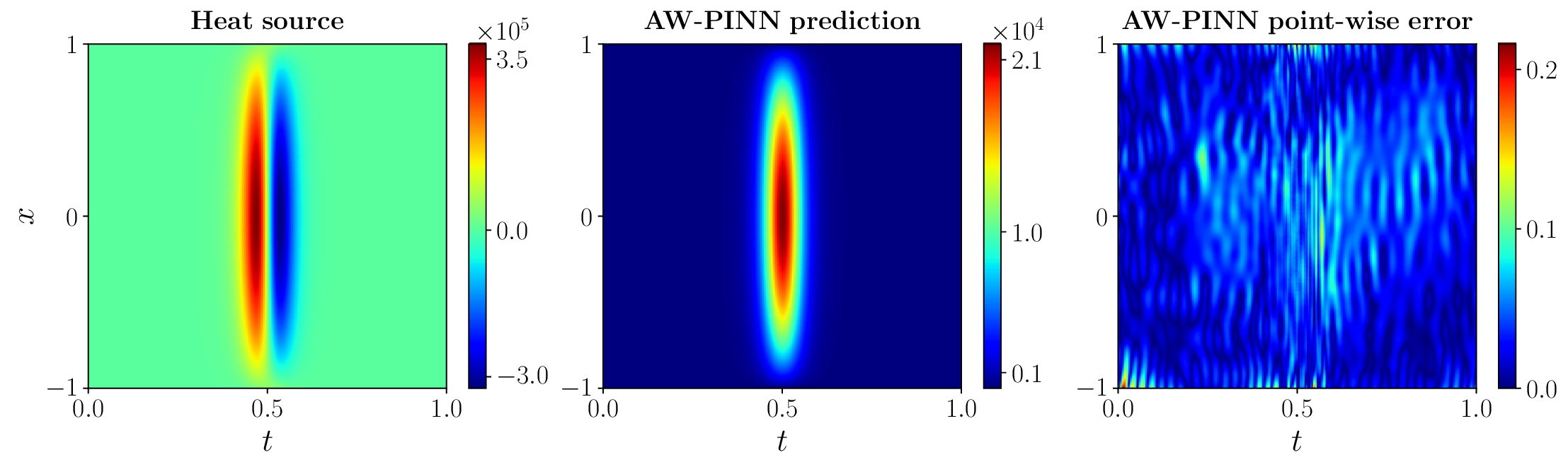}
    \caption{\textit{Left to Right:} Heat source function, the prediction using AW-PINN and respective absolute point-wise error for problem \ref{problem:1} with $\epsilon=0.1$.}
    \label{Fig:P1_plot}
\end{figure}

\begin{figure}[t!] 
    \centering
    \includegraphics[width=0.9\textwidth]{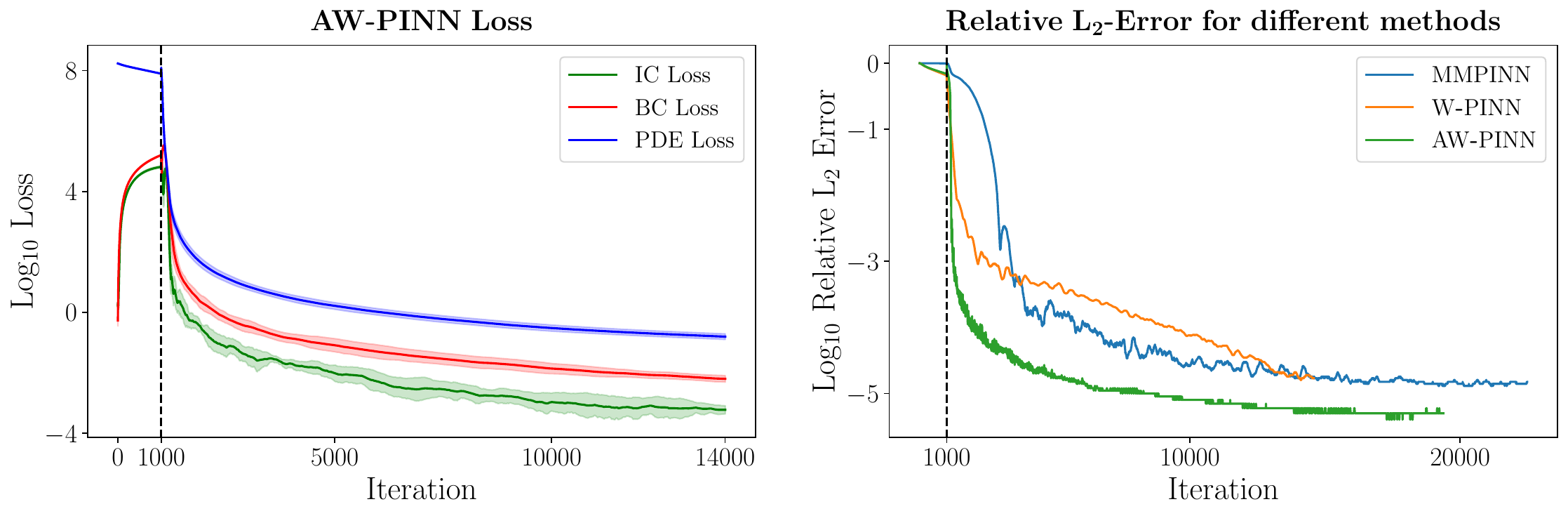}
    \caption{\textit{Left:} Average log loss plots for each loss terms in problem \ref{problem:1} with $\epsilon=0.12$ using AW-PINN. Shaded regions denote the standard deviation across $10$ independent runs. The vertical dashed line marks the completion of pre-training with W-PINN using $1000$ Adam iterations.\\
    \textit{Right:} Log relative $L_2$-error plots of different methods for problem \ref{problem:1} with $\epsilon=0.12$. The vertical dashed line indicates the end of the pre-training phase with $1000$ Adam iterations.}
    \label{Fig:P1_Loss}
\end{figure}

In our experiments, we employ $5000$ Adam iterations (approximately two minutes of pre-training) using the W-PINN for pre-training, followed by L-BFGS optimization with AW-PINN until convergence. Table \ref{Table:P1} lists comparative results for several small $\epsilon$ values. It includes the relative $L_2$-error, average training time, and number of collocation points used in training. Training times are omitted for methods with ineffective convergence. The table demonstrates AW-PINN achieves the lowest relative $L_2$-error for all $\epsilon$ values, with substantially reduced training time compared to MMPINN. Specifically for $\epsilon=0.1$, only AW-PINN could solve the problem satisfactorily. Another interesting observation is that, unlike MMPINN, AW-PINN does not require additional collocation points with a decrease in $\epsilon$ values. Furthermore, for $\epsilon=0.1$, extremely transient behavior of the heat source is evident from Fig. \ref{Fig:P1_plot}, and the consistently low point-wise absolute error across the domain demonstrates the effectiveness of the AW-PINN. Moreover, Fig. \ref{Fig:P1_Loss} shows the loss curve of AW-PINN and also compares the relative $L_2$-error with other methods for $\epsilon=0.12$. The narrow shaded regions for each loss term suggest low variance in the AW-PINN model, while the relative $L_2$-error comparison clearly demonstrates the faster convergence achieved by AW-PINN. Additionally, the non-oscillatory error curve indicates that AW-PINN provides more stable training dynamics than MMPINN.


\begin{table}[b!]
\centering
\caption{Comparison of various methods for solving the Poisson problem~\ref{problem:2}.}
\label{Table:P2}
\small
{\begin{tabular}{ccccccc}
\hline
$\epsilon$ & Method & $N_B$ & $N_R$ & Relative $L_2$-Error & Avg. Training Time \\
\hline \\

 & W-PINN & 4000 & 10000 & $3.55 \pm 0.26 \times 10^{-4}$ & 4.87 min \\
 0.05   & MMPINN & 6000 & 30000 & $5.71 \pm 1.53 \times 10^{-4}$ & 7.17 min \\
& \textbf{AW-PINN} & \textbf{4000} & \textbf{10000} & $\mathbf{3.42 \pm 0.13 \times 10^{-5}}$ & \textbf{5.14 min} \\

\\

 & W-PINN & 4000 & 15000 & $9.81 \pm 1.72 \times 10^{-3}$ & 4.21 min \\
 0.02   & MMPINN & 8000 & 50000 & $3.25 \pm 1.16 \times 10^{-3}$ & 9.11 min \\
& \textbf{AW-PINN} & \textbf{4000} & \textbf{10000} & $\mathbf{2.68 \pm 0.53 \times 10^{-4}}$ & \textbf{6.41 min} \\

\hline
\end{tabular}
}
\end{table}

\subsection{Poisson problem with highly localized source term}\label{problem:2}
In this test case, we examine a two-dimensional Poisson equation with an extremely localized source term. Such equations frequently arise in modeling electrostatic potentials with point charges and concentration profiles in chemical diffusion processes. They are also relevant in fields such as materials science and geophysics, where localized phenomena must be accurately resolved. Consider the following Poisson model:

{ \begin{equation}
        \begin{cases}
             \frac{\partial^2 u}{\partial x^2} + \frac{\partial^2 u}{\partial y^2}=f(x,y),\quad(x,y)\in(0,1)\times(0,1), \\[6pt]
            \mathscr{B}(x,y) = g(x,y),\quad (x,y)\in\partial \Omega,
        \end{cases}
\end{equation}}
with the exact solution as

\begin{flushleft}
\(
    \quad \quad \quad u(x,y) = 1+(y^2 + 10^3)\exp\left(-\frac{(x-0.5)^2}{2\epsilon^2}\right),
\)
\end{flushleft}
for small $\epsilon$ values, both the source function and the solution are highly localized, as illustrated in the top panel of Fig. \ref{Fig:P2_plot} for $\epsilon=0.02$.  This strong localization results in a huge loss imbalance. For example, at $\epsilon=0.02$ the initial loss ratio is $\mathcal{L}_b:\mathcal{L}_r=10^4:10^{11}$, making it a challenging scenario for the PINN-based methods.

Table \ref{Table:P2} compares the performance of various methods for $\epsilon=0.05$ and $\epsilon=0.02$. It demonstrates that the AW-PINN achieves relative $L_2$-error that is one and two orders of magnitude lower than those obtained by MMPINN and W-PINN, respectively. The same can be observed in Fig. \ref{Fig:P2_plot} bottom panel, where AW-PINN exhibits the lowest absolute point-wise error across the domain. Moreover, Fig. \ref{Fig:P2_scale} illustrates the scale adaptation capability of AW-PINN. In the case of W-PINN, it requires a highly resolved wavelet basis and corresponding matrices to handle such problems with localized high-scale features, which often need intensive memory and lead to highly non-convex optimization. Whereas AW-PINN efficiently adapts to higher scales only in the region where needed, without populating the entire domain with a high-resolution basis function. This not only improves computational efficiency but also enhances training stability and accuracy.

\begin{figure}[t!] 
    \centering
    \includegraphics[width=0.9\textwidth]{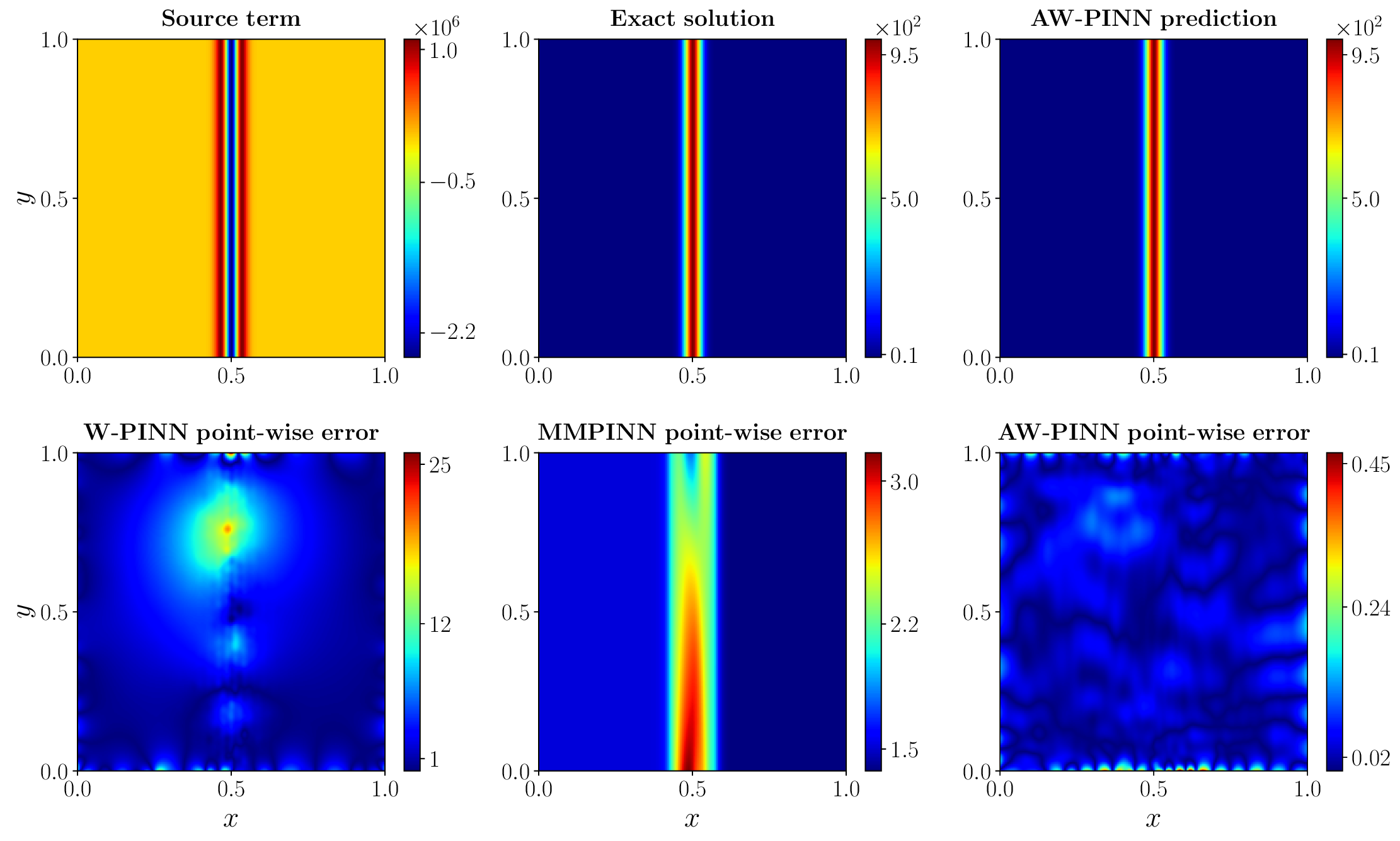} 
    \caption{\textit{Top - Left to Right:} The source function, exact solution, and prediction using AW-PINN for problem \ref{problem:2} with $\epsilon=0.02$.\\
    \textit{Bottom - Left to Right:} Absolute point-wise error of W-PINN, MMPINN, and AW-PINN, respectively.}
    \label{Fig:P2_plot}
\end{figure}

\begin{figure}[t!] 
    \centering
    \includegraphics[width=0.8\textwidth]{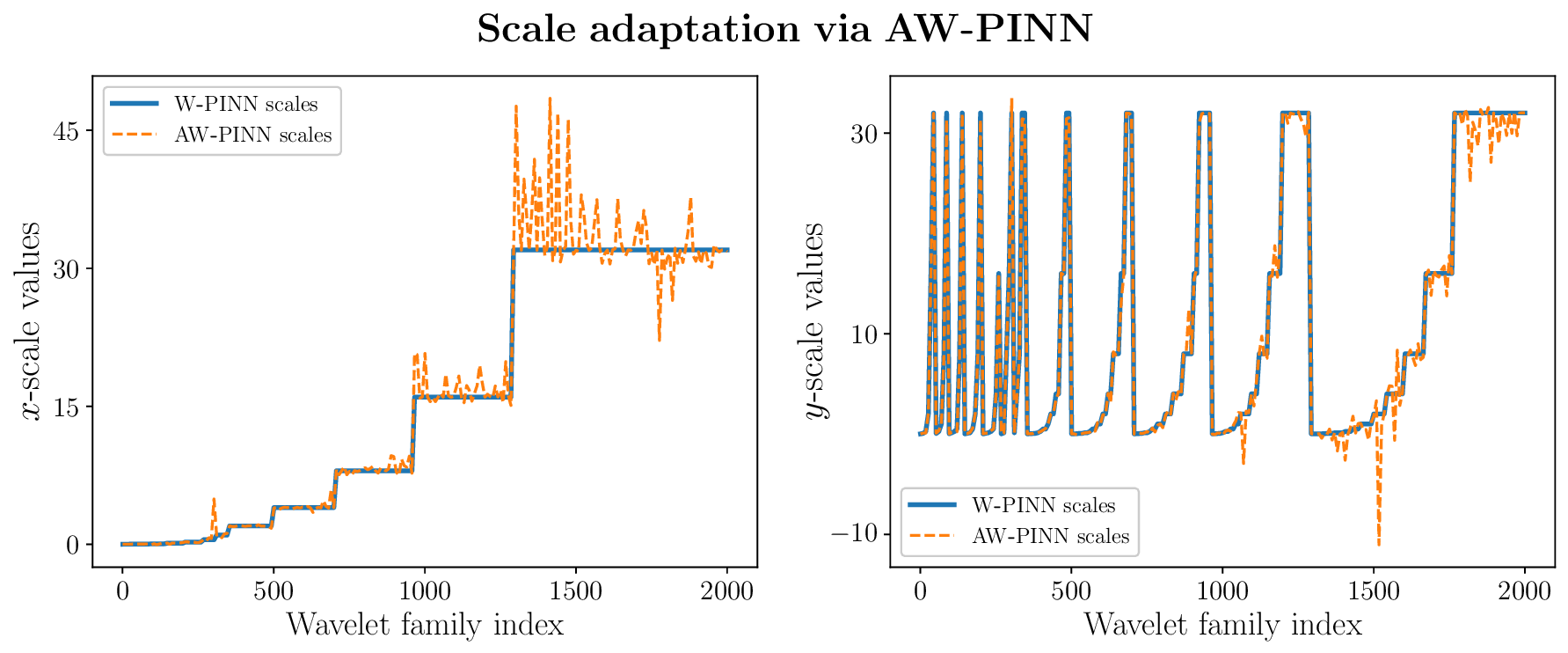} 
    \caption{Plot of x-scale (\textit left) and y-scale (\textit right) adaptation for problem \ref{problem:2}. Here, solid lines denote dyadic scales and dashed lines represent scales after adaptation.}
    \label{Fig:P2_scale}
\end{figure}


\subsection{Flow Equation with a strong oscillating source term}\label{problem:3}

The flow equation arises in a wide range of practical applications. For example, conservation laws in non-inertial reference frames, such as flows around airfoils, two-phase flow models incorporating momentum and energy exchange, turbulence models governed by two-equation formulations, and reactive flows involving chemical processes. For this example, we are interested in the following flow model with a large oscillating source term

{ \begin{equation}
        \begin{cases}
             \frac{\partial u}{\partial t} + \frac{\partial u}{\partial x} = A \sin \big(2\pi \frac{t}{T_s}\big),\quad(x,t)\in(-1,1)\times(0,1), \\[6pt]
            u(x,0) = \sin(\pi x),\quad x\in (-1,1),
        \end{cases}
\end{equation}}
with exact solution
\begin{flushleft}
\(
    \quad \quad \quad u(x,t) = \sin(\pi(x-t)) + \frac{AT_s}{2\pi}\Big(1-\cos\big(2\pi \frac{t}{T_s}\big)\Big),
\)
\end{flushleft}
here we take $A=100$ and $T_s = 0.05$. A small $T_s$ makes the source scale much smaller than the mean flow time scale, and a large $A$ strengthens the oscillations of the source term. These characteristics make this flow equation a challenging problem.

The comparison Table \ref{Table:P3} suggests that AW-PINN achieves accuracy superior to other methods by two orders of magnitude. The top panel of Fig. \ref{Fig:P3_sol} shows a highly oscillatory source term, AW-PINN prediction, and the corresponding absolute point-wise error, demonstrating the model's efficiency. The bottom panel further validates the accuracy of the prediction through multiple spatial snapshot plots.

\begin{table}[t!]
\centering
\caption{Comparison of various methods for solving the flow equation~\ref{problem:3}.}
\label{Table:P3}
\small
{\begin{tabular}{cccccccc}
\hline
Method & $N_I$ & $N_B$ & $N_R$ & Relative $L_2$-Error & Avg. Training Time \\
\hline \\

W-PINN & 1000 & 1000 & 20000 & $1.27 \pm 0.15 \times 10^{-2}$ & 4.17 min \\
MMPINN & 2000 & 2000 & 50000 & $8.35 \pm 0.97 \times 10^{-2}$ & 7.21 min \\
\textbf{AW-PINN} & \textbf{1000} & \textbf{1000} & \textbf{20000} & $\mathbf{5.17 \pm 0.61 \times 10^{-4}}$ & \textbf{6.17 min} \\

\hline
\end{tabular}
}
\end{table}

\begin{figure}[t!] 
    \centering
    \includegraphics[width=0.9\textwidth]{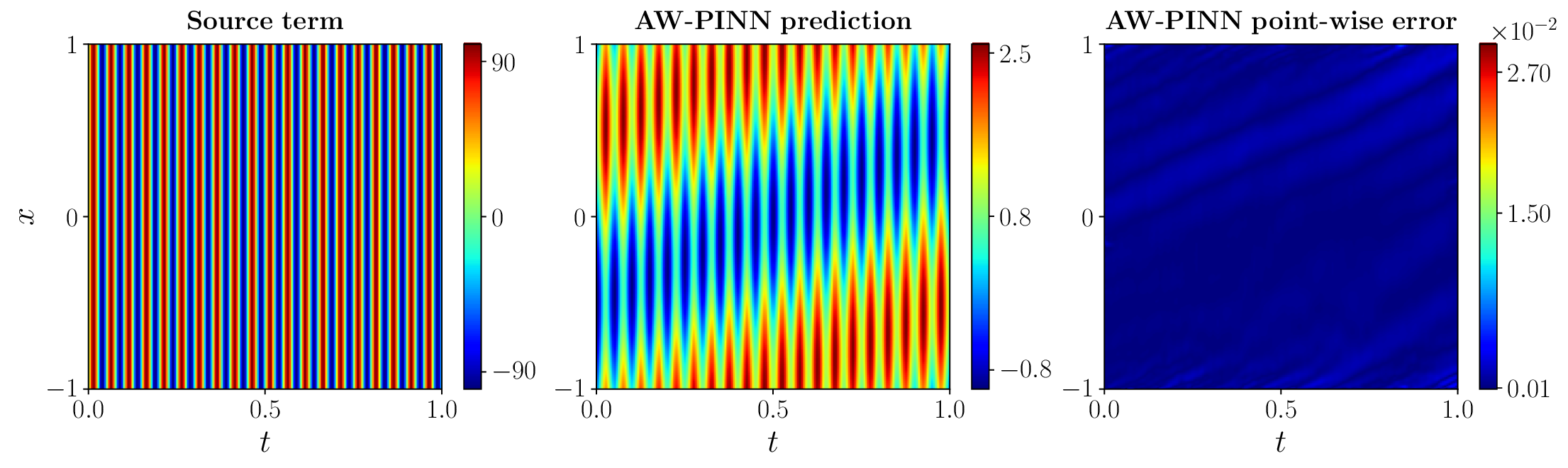} 

    \centering
    \includegraphics[width=0.7\textwidth]{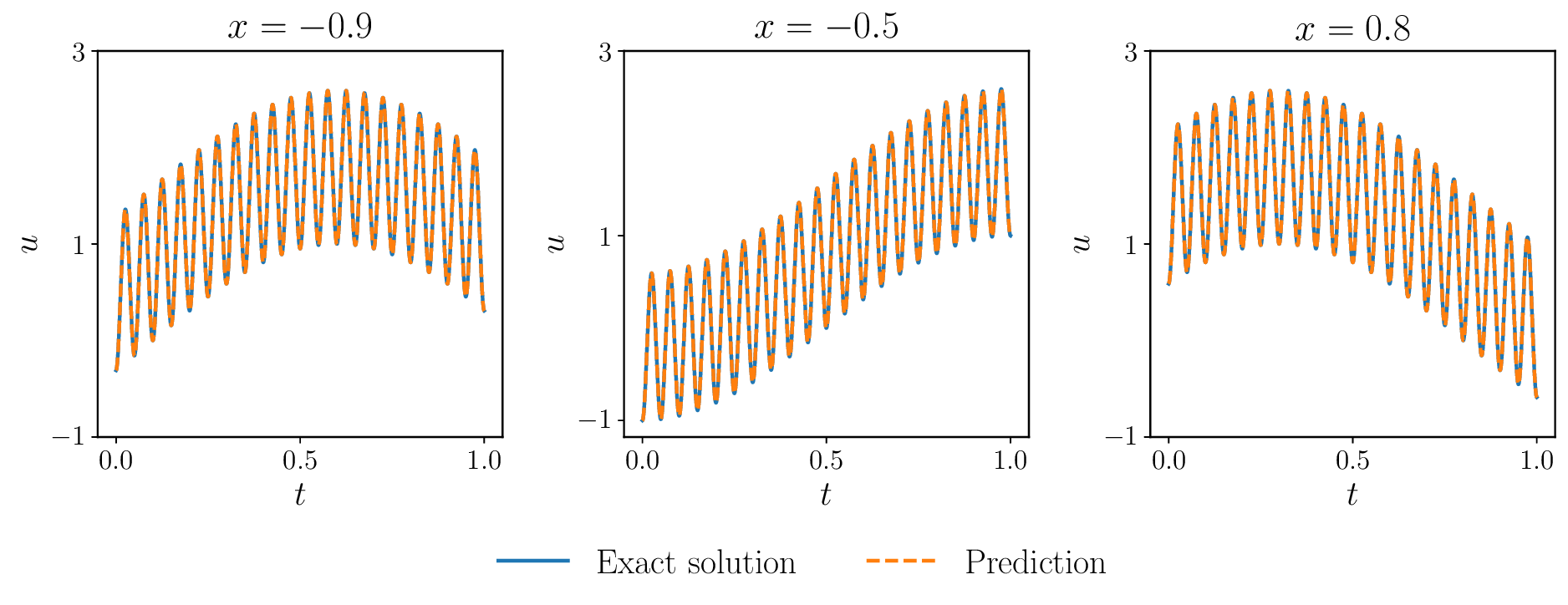} 
    \caption{\textit{Top - Left to Right:} The source function, the prediction using AW-PINN, and the corresponding absolute point-wise error for problem \ref{problem:3}.\\
    \textit{Bottom:} Cross-section comparison of the prediction with the exact solution at various $x$-domain snapshots.}
    \label{Fig:P3_sol}

\end{figure}


\subsection{Maxwell's equation for a point charge source}\label{problem:4}

In this example, we apply AW-PINN on a three-dimensional (one temporal and two spatial) electromagnetic problem. The fundamental principles of electromagnetism are governed by Maxwell's equations. It models the electric and magnetic fields through a set of coupled partial differential equations. These equations are widely used in various fields of physics and engineering, including wireless communications and antenna design, microwave devices, radar and remote sensing, photonics and optical waveguides, and biomedical imaging.

To validate the robustness of the proposed method, we consider the TEz formulation of Maxwell’s equations in a rectangular cavity with perfectly electrically conducting (PEC) boundary conditions

\begin{equation}
    \begin{aligned}
            \frac{\partial E_x}{\partial t} &= \frac{1}{\epsilon_o} \frac{\partial H_z}{\partial y}, \\
            \frac{\partial E_y}{\partial t} &= -\frac{1}{\epsilon_o} \frac{\partial H_z}{\partial x}, \\
            \frac{\partial H_z}{\partial t} &= -\frac{1}{\mu_o} \left( \frac{\partial E_y}{\partial x} - \frac{\partial E_x}{\partial y} + S \right),
    \end{aligned}
\end{equation}
where $E_x$ and $E_y$ denote the in-plane electric field components and $H_z$ denotes the out-of-plane magnetic field component. The constants $\epsilon_o$ and $\mu_o$ are the permeability and permittivity of the space, respectively. The source function $S$ represents a Gaussian pulse emitted by a point source located at $(x_o,y_o)$, which can be expressed as

\begin{equation}
    S(x,y,t) = \delta(x-x_o)\delta(y-y_o)\cdot \exp\Big(-\Big(\frac{t-\tau}{\omega}\Big)^2\Big),
\end{equation}
with temporal delay $\tau$ and pulse width $\omega$. The spatial point source is approximated using the Gaussian distribution centered at $(x_o,y_o)$ with kernel width $\sigma=0.01$. 
\begin{equation}
   \delta(x - x_0)\, \delta(y - y_0) \approx \frac{1}{2\pi \sigma^2} \exp\left( -\frac{(x-x_0)^2 + (y-y_0)^2}{2\sigma^2} \right). 
\end{equation}

\begin{table}[b!]
\centering
\caption{Comparison of various methods for solving Maxwell's equation~\ref{problem:4}.}
\label{Table:P4}
\small
{\begin{tabular}{cccccccc}
\hline
Method & Relative $L_2$-Error $E_x$ & Relative $L_2$-Error $E_y$ & Relative $L_2$-Error $H_z$ & Avg. Training Time \\
\hline \\

W-PINN & $8.52 \pm 0.21 \times 10^{-2}$ & $8.94 \pm 0.13 \times 10^{-2} $ & $4.52 \pm 0.21 \times 10^{-1}$ & 42.71 min \\
MMPINN & $4.33 \pm 0.17 \times 10^{-2}$ & $3.91 \pm 0.21 \times 10^{-2}$ & $2.31 \pm 0.18 \times 10^{-1}$ & 182.16 min \\
\textbf{AW-PINN} & $\mathbf{
5.31 \pm 0.61 \times 10^{-3}
}$ & $\mathbf{
6.54 \pm 0.47 \times 10^{-3}
}$ & $\mathbf{9.01 \pm 0.83 \times 10^{-3}}$ & \textbf{76.57 min} \\

\hline
\end{tabular}
}
\end{table}

\begin{figure}[t!] 
    \centering
    \includegraphics[width=0.9\textwidth]{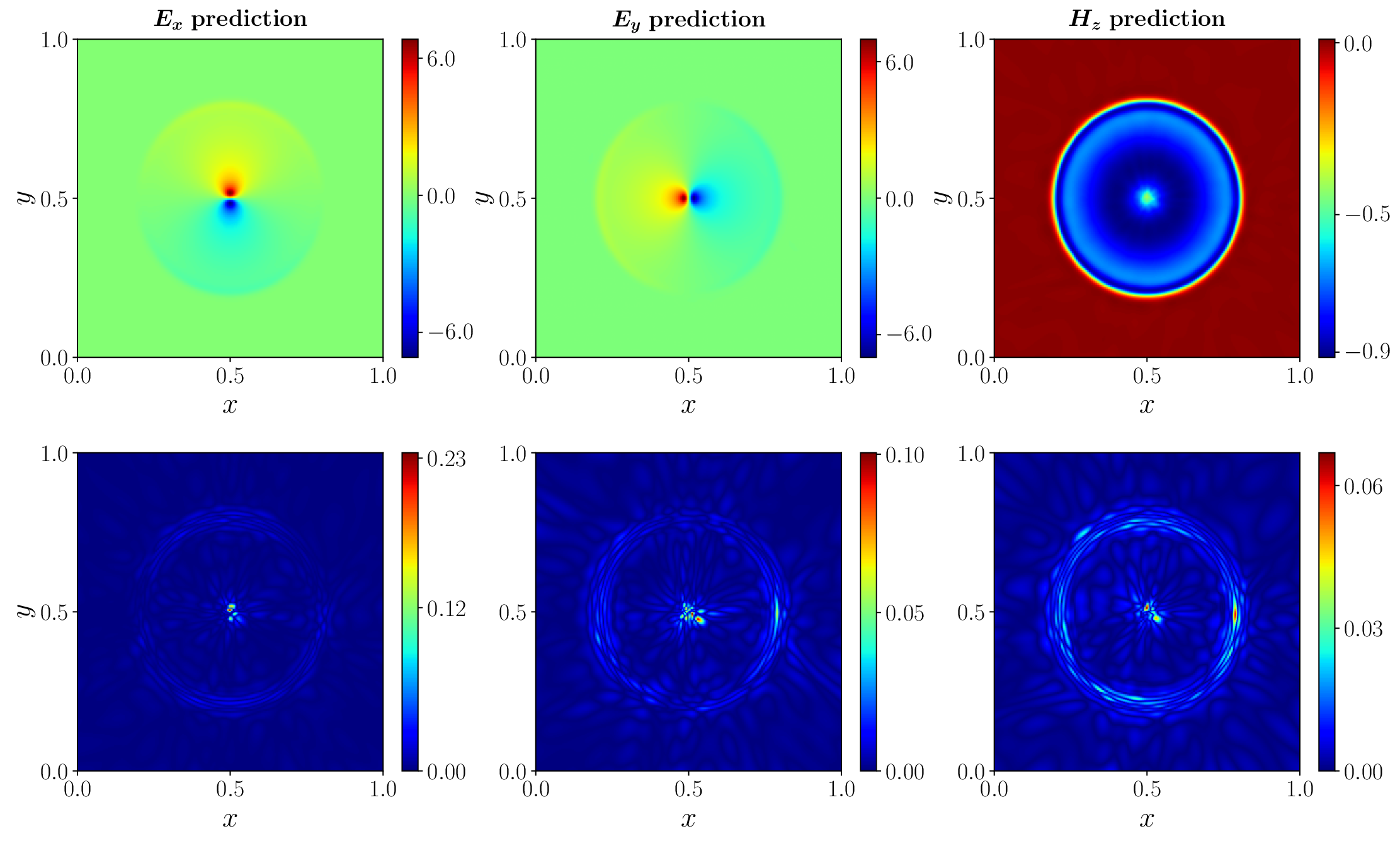} 
    \caption{The electromagnetic field prediction using AW-PINN for TEz Maxwell's equations \ref{problem:4} at the top panel and their respective point-wise absolute error at the bottom panel.}
    \label{Fig:P4_plot}
\end{figure}

To solve this system of equations, all physical quantities are non-dimensionalized such that $\epsilon_o$ and $\mu_o$ become unity, and take $\tau=0.25$ and $\omega=0.25$. We solve the problem for the domain $\Omega=[0,1]^2$ with the point source located at $(x_o,y_o) = (0.5,0.5)$ and starting with a null field condition. The simulation is run for the non-dimensional time $T=0.5$. For a reference solution, we employ the finite-difference time-domain (FDTD) method with spatial resolution as $\Delta x = \Delta y = 5\times10^{-3}$ and time step $\Delta t=1.5\times10^{-3}$, chosen to satisfy the CFL stability condition under the non-dimensional wave speed.

W-PINN and MMPINN are unable to solve this problem satisfactorily, which is evident from high relative $L_2-$error values in the Table \ref{Table:P4}. Further, the point-wise error plot shown in Fig. \ref{Fig:P4_plot} demonstrates higher accuracy of AW-PINN prediction.


\section{Conclusion} \label{sec: conclude}

In this study, we addressed the challenges encountered during the training of PINNs due to the severe loss imbalance between different loss components. In particular, we considered problems with localized high-magnitude source terms. For such problems, there is an extreme loss imbalance, and it requires special treatment. To tackle these issues, we proposed an extension of W-PINN that dynamically adapts the wavelet family based on residual and supervised loss functions. This adaptive mechanism enables the model to effectively capture fine-scale features without populating the entire domain with higher-resolution wavelets. Through various challenging problems, we have demonstrated significant improvement over W-PINN and the well-known MMPINN.

Although the method is promising and has achieved impressive accuracy with highly competitive training time, there is still room for improvement. The similarity-based selection of wavelet bases involves manual tuning partially and depends on the quality of the pre-training phase, suggesting that more automated selection schemes could improve robustness. Additionally, systematically investigating the performance of AW-PINN using various wavelet bases could be an interesting comparison. Moreover, future work could involve evaluating the AW-PINN across a broader class of multiscale problems to further assess its general applicability and performance.

\subsubsection*{Acknowledgment}
The authors would like to thank Mr. Anirudha Sen, Department of Mechanical Engineering, Indian Institute of Technology Bhilai, India, for his assistance in computations related to Neural Tangent Kernel theory for Physics-Informed Neural Networks.

\noindent 

 
\subsubsection*{Funding details}
\noindent No funding source is available for this research.

\subsubsection*{Data availability statements}
All source code to reproduce the results of this study will be made available on request.

\subsubsection*{Conflicts of interest and declarations}	
The authors declare that they do not have any conflicts of interest. In addition, they also declare that this work is not under consideration anywhere.

		
		
	
\bibliography{ReferencesSS}


\newpage
\appendix
\renewcommand{\thetable}{A.\arabic{table}} 
\setcounter{table}{0}   
\section*{Appendix A: Hyperparameters}\label{A1}


\begin{table}[h!]
\centering
\caption{Hyperparameters used for MMPINN}
\resizebox{\textwidth}{!}{
\begin{tabular}{lcccccccc}
\toprule
\textbf{Hyperparameter} 
& \multicolumn{3}{c}{\textbf{Heat Conduction \ref{problem:1}}} 
& \multicolumn{2}{c}{\textbf{Poisson Problem \ref{problem:2}}} 
& \textbf{Flow Eq. \ref{problem:3}} 
& \textbf{Maxwell's Eq. \ref{problem:4}} \\ 
\cmidrule(lr){2-4} \cmidrule(lr){5-6}

& $\epsilon = 0.12$ & $\epsilon = 0.11$ & $\epsilon = 0.10$
& $\epsilon = 0.05$ & $\epsilon = 0.02$ &  &  \\ 
\midrule

Hidden layers           & 5 & 5 & 6 & 5 & 5 & 5 & 8 \\
Neurons per layer       & 400 & 400 & 500 & 200 & 300 & 300 & 400 \\
Residual points $N_{res}$ & 90000 & 100000 & 200000 & 30000 & 50000 & 50000 & 150000 \\
Supervised points $N_{sup}$ & 9000 & 12000 & 21000 & 6000 & 8000 & 4000 & 20000 \\
Exponent $p$            & 3.0 & 4.0 & 4.0 & 3.0 & 4.0 & 3.0 & 2.0 \\
Exponent $q$            & 1.0 & 1.0 & 1.0 & 1.0 & 1.0 & 1.0 & 1.0 \\
ADAM iterations         & 1000 & 2000 & 4000 & 2000 & 2000 & 4000 & 10000 \\
\bottomrule
\end{tabular}}
\label{tab:mmpinn_hyper}
\end{table}

\vspace{0.5em}

\begin{table}[h!]
\centering
\caption{Hyperparameters used for W-PINN}
\resizebox{\textwidth}{!}{
\begin{tabular}{lcccccccc}
\toprule
\textbf{Hyperparameter} 
& \multicolumn{3}{c}{\textbf{Heat Conduction \ref{problem:1}}} 
& \multicolumn{2}{c}{\textbf{Poisson Problem \ref{problem:2}}} 
& \textbf{Flow Eq. \ref{problem:3}} 
& \textbf{Maxwell's Eq. \ref{problem:4}} \\ 
\cmidrule(lr){2-4} \cmidrule(lr){5-6}

& $\epsilon = 0.12$ & $\epsilon = 0.11$ & $\epsilon = 0.10$
& $\epsilon = 0.05$ & $\epsilon = 0.02$ &  &  \\ 
\midrule

Hidden layers           & 6 & 6 & 6 & 9 & 9 & 6 & 10 \\
Neurons per layer       & 100 & 100 & 100 & 50 & 50 & 100 & 200 \\
Residual points $N_{res}$ & 20000 & 20000 & 20000 & 10000 & 15000 & 20000 & 40000 \\
Supervised points $N_{sup}$ & 3000 & 3000 & 3000 & 4000 & 4000 & 2000 & 10000 \\
Residual loss weight $\omega_r$   & 0.01 & 0.01 & 0.01 & 0.01 & 0.01 & 1.0 & 0.1 \\
Supervised loss weight $\omega_s$   & 1.0 & 1.0 & 1.0 & 1.0 & 1.0 & 1.0 & 1.0 \\
Resolution $J_x$            & [-6,5] & [-6,5] & [-6,5] & [-6,6] & [-6,6] & [-6,6] & [-5,4] \\
Resolution $J_y$            & - & - & - & [-6,6] & [-6,6] & - & [-5,4] \\
Resolution $J_t$            & [-6,6] & [-6,6] & [-6,6] & - & - & [-6,6] & [-5,4] \\
Translation parameter $\gamma$   & 0.3 & 0.3 & 0.3 & 0.1 & 0.2 & 0.3 & 0.1 \\
ADAM iterations         & 1000 & 2000 & 2000 & 1000 & 2000 & 3000 & 10000 \\
\bottomrule
\end{tabular}}
\label{tab:mmpinn_hyper}
\end{table}

\begin{table}[h!]
\centering
\caption{Hyperparameters used for AW-PINN}
\resizebox{\textwidth}{!}{
\begin{tabular}{lcccccccc}
\toprule
\textbf{Hyperparameter} 
& \multicolumn{3}{c}{\textbf{Heat Conduction \ref{problem:1}}} 
& \multicolumn{2}{c}{\textbf{Poisson Problem \ref{problem:2}}} 
& \textbf{Flow Eq. \ref{problem:3}} 
& \textbf{Maxwell's Eq. \ref{problem:4}} \\ 
\cmidrule(lr){2-4} \cmidrule(lr){5-6}

& $\epsilon = 0.12$ & $\epsilon = 0.11$ & $\epsilon = 0.10$
& $\epsilon = 0.05$ & $\epsilon = 0.02$ &  &  \\ 
\midrule

Residual points $N_{res}$ & 20000 & 20000 & 20000 & 10000 & 10000 & 20000 & 40000 \\
Supervised points $N_{sup}$ & 3000 & 3000 & 3000 & 4000 & 4000 & 2000 & 10000 \\
Pre-trainingADAM iterations  & 1000 & 2000 & 3000 & 1000 & 2000 & 1000 & 10000 \\
Residual loss weight $\omega_r$   & 0.01 & 0.01 & 0.01 & 0.01 & 0.01 & 1.0 & 0.1 \\
Supervised loss weight $\omega_s$   & 1.0 & 1.0 & 1.0 & 10.0 & 10.0 & 1.0 & 1.0 \\
Threshold IC            & $> 0.5$ & $> 0.2$ & $> 0.4$ & - & - & $> 0.5$ & $< 0.95$ \\
Threshold BC   & $< 0.002$ & $< 0.02$ & $< 0.003$ & $> 0.2$ & $> 0.2$ & $>0.95$ & $< 0.95$ \\
Threshold PDE  & $< 0.980$ & $< 0.988$ & $< 0.984$ & $< 0.850$ & $< 0.975$ & $<0.25$ & $< 0.91$ \\
Top $\kappa\,\text{-}\,\%$            & 8 & 10 & 8 & 2 & 1 & 2 & 5 \\
\bottomrule
\end{tabular}}
\label{tab:mmpinn_hyper}
\end{table}

\end{document}